\documentclass[letterpaper, 10 pt, conference]{ieeeconf}  

\IEEEoverridecommandlockouts                              

\title{\LARGE \bf
 Tube RRT*: Efficient Homotopic Path Planning for Swarm Robotics Passing-Through Large-Scale Obstacle Environments
}

\author{Pengda Mao$^{1}$, Shuli Lv$^{1}$, and Quan Quan$^{1}$~\IEEEmembership{Senior Member,~IEEE}
	\thanks{$^{1}$Pengda Mao, Shuli Lv, and Quan Quan are with  School of Automation Science and Electrical Engineering,
		Beihang University, Beijing, 100191, P.R. China (e-mail: maopengda@buaa.edu.cn; lvshuli@buaa.edu.cn;
		qq\_buaa@buaa.edu.cn).
	}%
}

\usepackage{pdfpages}
\usepackage{graphicx}
\usepackage{booktabs}
\usepackage{multirow}
\usepackage{amssymb,amsmath,bm,color,enumerate,mathtools,fancyhdr,graphics}
\usepackage{mathrsfs}
\usepackage{lipsum}
\usepackage{changepage}
\usepackage{subfigure}
\usepackage{stfloats}
\usepackage{setspace}
\usepackage{stfloats}
\usepackage[colorlinks,linkcolor=blue]{hyperref}
\hypersetup{hidelinks,
	colorlinks=true,
	allcolors=black,
	pdfstartview=Fit,
	breaklinks=true}
\usepackage{algorithmic}
\usepackage{algorithm}

\usepackage{array}

\newtheorem{Def}{Definition}
\newtheorem{The}{Theorem}

\newtheorem{prop}{Proposition}
\newtheorem{lemma}{Lemma}

\UseRawInputEncoding
\begin{document}

\maketitle
\thispagestyle{empty}
\pagestyle{empty}

\begin{abstract}
Recently, the concept of homotopic trajectory planning has emerged as a novel solution to navigation in large-scale obstacle environments for swarm robotics, offering a wide ranging of applications. However, it lacks an efficient homotopic path planning method in large-scale obstacle environments. This paper introduces Tube RRT*, an innovative homotopic path planning method that builds upon and improves the Rapidly-exploring Random Tree (RRT) algorithm. Tube RRT* is specifically designed to generate homotopic paths, strategically considering gap volume and path length to mitigate swarm congestion and ensure agile navigation. Through comprehensive simulations and experiments, the effectiveness of Tube RRT* is validated.
\end{abstract}

\section{INTRODUCTION}
In recent years, swarm robotics has emerged as a promising application in various fields, including search and rescue operations, environmental monitoring, agriculture, exploration, and logistics. A current focus is on determining the safe, reliable, and smooth movement of swarm robotics within large-scale obstacle environments.

In addressing the issue of robot swarms passing through large-scale obstacle environments, there are currently numerous methods available, including trajectory planning \cite{zhou2022swarm}, control-based methods \cite{vasarhelyi2018optimized}, and virtual tube approaches \cite{Quan2021Distributed}. The swarm trajectory planning employs a hierarchical structure, which enables the smooth movement of the swarm. However, when the number of robots is large and the environment is complex, the optimization problems of the trajectory planning become complex and difficult to solve, leading to high computational load and frequent replanning \cite{usenko_real-time_2017,ding_efficient_2019}. Control-based methods, on the other hand, have stability, simplicity, and low computation cost. However, they lack prediction and the simplicity of robot modeling leads to less smooth movement and susceptibility to local minima. By combining the strengths of these two methods, the optimal virtual tube method proposed in our previous work \cite{mao2023optimal} is a method suitable for large-scale swarm movement. It confines the robot swarm in free space and achieves low computation cost to pass through obstacle environments through centralized trajectory planning and distributed control.

\textcolor{blue}{Homotopic paths could serve as a front end to the above methods. 
	Homotopic paths in topology refer to a concept that illustrates the continuous deformation between two paths. 
	By ensuring that there are no obstacles between paths, homotopic paths decouple obstacle avoidance from inter-robot avoidance, significantly reducing the risk of deadlock. This decoupling simplifies the control strategy and increases efficiency. Therefore, an effective homotopic path planning method is essential to enable the swarm to navigate through obstacle environments.}

There are many works on finding homotopic paths to speed up the computation. A reference frame determining the topological relationship between obstacles is used to compute a topological graph for restriction criteria of homotopy classes \cite{hernandez2015comparison}. Both the Homotopic RRT (HRRT) \cite{kim2003motion} and the Homotopic A* (HA*) \cite{hart1968formal} explore the direction in complex space by checking the intersections with the reference frame to generate the homotopic paths. However, these homotopic path planning methods focus on accelerating the computation speed of a single robot to find a single optimal path \cite{yi2016homotopy,liu2023homotopy}. Even in swarm path planning problem, due to the lack of coordination among agents in the swarm, the planning of homotopic paths is limited to a single robot exploring a collision-avoidance path of a specific shape, such as an elliptical homotopy path \cite{fu2023ftsa}.

\textcolor{blue}{Previous studies \cite{Quan2021Distributed,mao2023optimal} of virtual tubes primarily focus on control and homotopic trajectory planning within the virtual tube but lack research on homotopic path planning. Current homotopic path planning methods \cite{hernandez2015comparison,kim2003motion,fu2023ftsa} are designed for single robots and do not consider the impact of gap volume on the swarm. Therefore, applying single-robot homotopic path planning methods is time-consuming and leads to sudden contractions, resulting in congestion within the swarm.}
Thus, in this paper, an efficient infinite homotopic path planning algorithm for swarm is proposed. The main contributions are as follows:
\begin{itemize}
	\item {A novel homotopic path planning algorithm, named Tube RRT*, is proposed to consider the path length and the volume of gaps synchronously, which generates infinite homotopic paths for swarm robotics.}
	\item \textcolor{blue}{The Tube RRT* centrally plans $l$ homotopic paths with a computational complexity of $O\left(n(\text{log}n+1)+l\right)$ for swarm robotics, compared to $O\left(nl(\text{log}n+1)\right)$ of the RRT*. }
	\item Theorems and proposition demonstrate that the proposed algorithm tends to find homotopic paths with the fewest path points and equal cross-sectional volumes, given a fixed total volume and path length.
	\item The effectiveness of the proposed algorithm is validated through various comparisons in simulations and experiments.
\end{itemize}

\section{Preliminaries and Problem Formulation}
In this section, the homotopic paths are introduced based on the related concepts of the virtual tube. Then, a homotopic path planning problem is described.
\subsection{Preliminaries}
\begin{Def}[Virtual Tube \cite{mao2023optimal}]
		A \emph{virtual tube} $\mathcal T$, as shown in Fig. \ref{fig:virtual-tube}(a), is a set in a configuration space $X$ represented by a 4-tuple $\left( {\mathcal C}_0, {\mathcal C}_1, {\bf f}, {\bf h} \right)$ where 
	\begin{itemize}
		\item ${\mathcal C}_0, {\mathcal C}_1$, called \emph{terminals}, are disjoint bounded convex subsets in $n$-dimension space.
		\item ${\bf f}$ is a \emph{diffeomorphism}: ${\mathcal C}_0 \to {\mathcal C}_1$, so that there is a set of order pairs ${\mathcal{P}} = \left\{ {\left( {{\bf q}_0,{\bf q}_m} \right)| {\bf q}_0 \in {\mathcal{C}_0},{\bf q}_m = {\bf f}\left( {\bf q}_0 \right) \in {\mathcal{C}}_1} \right\}$.
		\item ${\bf h}$ is a smooth\footnote{A real-valued function is said to be smooth if its derivatives of all orders exist and are continuous.} map: ${\mathcal P} \times {\mathcal I} \to {\mathcal T} $ where ${\mathcal I} = [0,1]$, such that ${\mathcal T} = \{ {\bf h} \left( \left({\bf q}_0,{\bf q}_m\right),t\right) | \left({\bf q}_0,{\bf q}_m\right) \in {\mathcal P} , t \in {\mathcal I} \}$, ${\bf h}\left( {\left( {\bf q}_0,{\bf q}_m \right),0} \right) = {\bf q}_0$, ${\bf h}\left( {\left( {\bf q}_0,{\bf q}_m \right),1} \right) = {\bf q}_m$. The function ${\bf h}\left( {\left( {\bf q}_0,{\bf q}_m \right),t} \right)$ is called a \emph{trajectory} for an order pair $\left( {\bf q}_0,{\bf q}_m\right)$.
	\end{itemize}
	\label{def:virtual-tube}
\end{Def}
\begin{figure}
	\centering
	\includegraphics{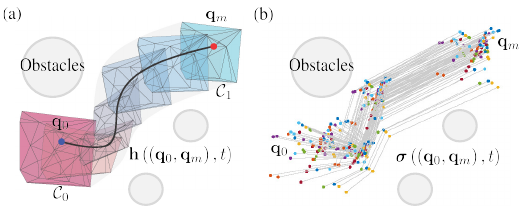}
	\caption{Virtual tube in obstacle environments. (a) The purple and blue polyhedrons are terminals. The black curve is the trajectory from ${\bf q}_0\in{\mathcal{C}_0}$ to ${\bf q}_m \in \mathcal{C}_1$. (b) The colorful points denote path points of the homotopic paths and the gray lines represent homotopic paths $\boldsymbol{\sigma}\left(\left({\bf q}_0,{\bf q}_m\right),t\right)$.}
	\label{fig:virtual-tube}
	\vspace{-0.7cm}
\end{figure}

\begin{Def}[Path]
	For any order pair $\left( {{{\bf{q}}_{0}},{{\bf{q}}_{m}}} \right)\in \mathcal{P}$, the \emph{path} is a continuous function $\boldsymbol{\sigma}\left(\left({\bf q}_0,{\bf q}_m\right),t\right) :{\mathcal P} \times {\mathcal I} \to { X}$ where $\boldsymbol{\sigma}\left(\left({\bf q}_0,{\bf q}_m\right),0\right)={\bf q}_0$ and $\boldsymbol{\sigma}\left(\left({\bf q}_0,{\bf q}_m\right),1\right)={\bf q}_m$, $m$ represents the number of the path points. And, the paths $\boldsymbol{\sigma}\left(\left({\bf q}_{0,k},{\bf q}_{m,k}\right),t\right)$ is the boundary paths if the points ${\bf q}_{0,k}$ and ${\bf q}_{m,k}$ are the vertexes of the terminals $\mathcal{C}_0$ and $\mathcal{C}_1$ respectively.
\end{Def}
\begin{Def}[Homotopic Paths]
	For any ${\boldsymbol{{\sigma}}}_1,\boldsymbol{{\sigma}}_2 \in \Sigma_{\boldsymbol{\sigma}}$, the paths ${\boldsymbol{{\sigma}}}_1$ and ${\boldsymbol{{\sigma}}}_2$, as shown in Fig. \ref{fig:virtual-tube}(b), are called \emph{homotopic paths} if there is a continuous map ${\bf H}:\mathcal{I} \times \mathcal{I}\to X$ such that ${\bf{H}}\left( {t,0} \right) = {\boldsymbol{\sigma} _1}\left( {\left( {{{\bf{q}}_{0,1}},{{\bf{q}}_{m,1}}} \right),t} \right)$ $,{\bf{H}}\left( {t,1} \right) = {\boldsymbol{\sigma} _2}\left( {\left( {{{\bf{q}}_{0,2}},{{\bf{q}}_{m,2}}} \right),t} \right)$, ${\bf{H}}\left( {0,0} \right) = {{\bf{q}}_{0,1}}$, ${\bf{H}}\left( {1,0} \right) = {{\bf{q}}_{m,1}}$, ${\bf{H}}\left( {0,1} \right) = {{\bf{q}}_{0,2}}$$,{\bf{H}}\left( {1,1} \right) = {{\bf{q}}_{m,2}}$.
	\label{def:homo-paths}
\end{Def}

It should be noted that the definition of homotopic paths in \emph{Definition \ref{def:homo-paths}} is a generalization of those defined in other works \cite{fu2023ftsa,osa2022motion}, without fixed start and goal points to accommodate the requirements of swarm.

\begin{Def}[Path Length]
	The path length of the homotopic paths $\Sigma_{\boldsymbol{\sigma}}$ is the length of the center path ${\boldsymbol{\sigma}_o}\in\Sigma_{\boldsymbol{\sigma}} $, which is expressed as ${L}\left( \boldsymbol{\sigma}_o  \right)$. The center paths $\boldsymbol{\sigma}_o$ is linear combination of the boundary paths $\boldsymbol{\sigma}_k\in\Sigma_{\boldsymbol{\sigma}} $, which is expressed as
	\begin{equation}
		 {{\boldsymbol{\sigma }}_o} = {M^{ - 1}}\sum\nolimits_{i = 1}^M {{{\boldsymbol{\sigma }}_i}} ,{{\boldsymbol{\sigma }}_i} \in {\boldsymbol{\sigma }},
	\end{equation}
where $M$ is the number of the boundary paths.
\label{def:tube-path}
\end{Def}

\subsection{Homotopic Path Planning Problem Formulation}
This work aims to develop a method to determine homotopic paths for the swarm robotics in 3-D space, considering both the path length and the opening volume. The problem is described as follows:

Let $ X \subset \mathbb{R}^3$ be the configuration space, $X_\text{obs}$ be the obstacle space. Thus, the free space is denoted as $X_\text{free} = X/X_\text{obs}$. The homotopic path planning algorithm aims to find the homotopic paths ${\sum _{{\boldsymbol{\sigma }}}} \subset  X_\text{free}$ with a center path $\boldsymbol{\sigma}_o$ such that
\begin{equation}
	\boldsymbol{\sigma}^*_o = \arg \mathop {\min }\limits_{\boldsymbol{\sigma}_o} {f_1}\left( \boldsymbol{\sigma}_o  \right) + {f_2}\left( \boldsymbol{\sigma}_o  \right),
\end{equation}
where $f_1,f_2: X \to \mathbb{R}_{\ge 0}$ are functions related to path length and gap volume respectively.
\section{Methods}
In this section, a new modified RRT* algorithm, Tube RRT*, is introduced. First, this new homotopic path planning algorithm is outlined in \emph{Algorithm \ref{alg:TubeRRT}}. Subsequently, some primitive procedures are defined. Finally, the analysis of properties of the proposed algorithm is presented.
\subsection{Outline}
The proposed algorithm shown in \emph{Algorithm \ref{alg:TubeRRT}} is divided into 6 steps. Step 1 (line 5-7, Fig. \ref{fig:scheme-of-tuberrt}(a)): Sample a random sphere\footnote{\textcolor{blue}{In the high-dimensional space, the n-dimension ellipsoid or convex polyhedra is more suitable than the sphere for describing free space. However, since this study considers the robot in swarm as the agent. Research in high-dimension spaces is beyond the scope of this paper.}} ${\bf x}_\text{rand}\in X_\text{free}$ centered in ${\bf q}_\text{rand}$ with the radius $r_\text{rand}$ which is the minimum distance to obstacles in free space, so that the sphere ${\bf x}_\text{rand}$ is denoted as ${\bf x}_\text{rand}=\left({\bf q}_\text{rand},r_\text{rand}\right)$. Step 2 (line 8, Fig. \ref{fig:scheme-of-tuberrt}(b)): Find the nearest sphere ${\bf x}_\text{nearest}$ to the random sphere ${\bf x}_\text{rand}$ in the tree $\mathcal{G}$. Step 3  (line 9-10, Fig. \ref{fig:scheme-of-tuberrt}(c)): Steer the random sphere ${\bf x}_\text{rand}$ towards the nearest sphere ${\bf x}_\text{nearest}$ to obtain a new sphere ${\bf x}_\text{new}$ which has an intersection between ${\bf x}_\text{nearest}$ and ${\bf x}_\text{new}$. Meanwhile, the radius $r_\text{new}$ of ${\bf x}_\text{new}$ remains consistent with the distance to obstacles, that is constantly changing and restricted by the minimum radius $r_\text{min}$. Step 4 (line 11, Fig. \ref{fig:scheme-of-tuberrt}(d)): Find a set $X_\text{near}$ that contains all spheres intersected with the new sphere ${\bf x}_\text{new}$ in the tree $\mathcal{G}$. Step 5 (line 12-13, Fig. \ref{fig:scheme-of-tuberrt}(e)): Rewire the tree to choose the connections among ${\bf x}_\text{new}$ and the spheres in the set $X_\text{near}$ with the minimum cost. 
\begin{figure*}
	\centering
	\includegraphics{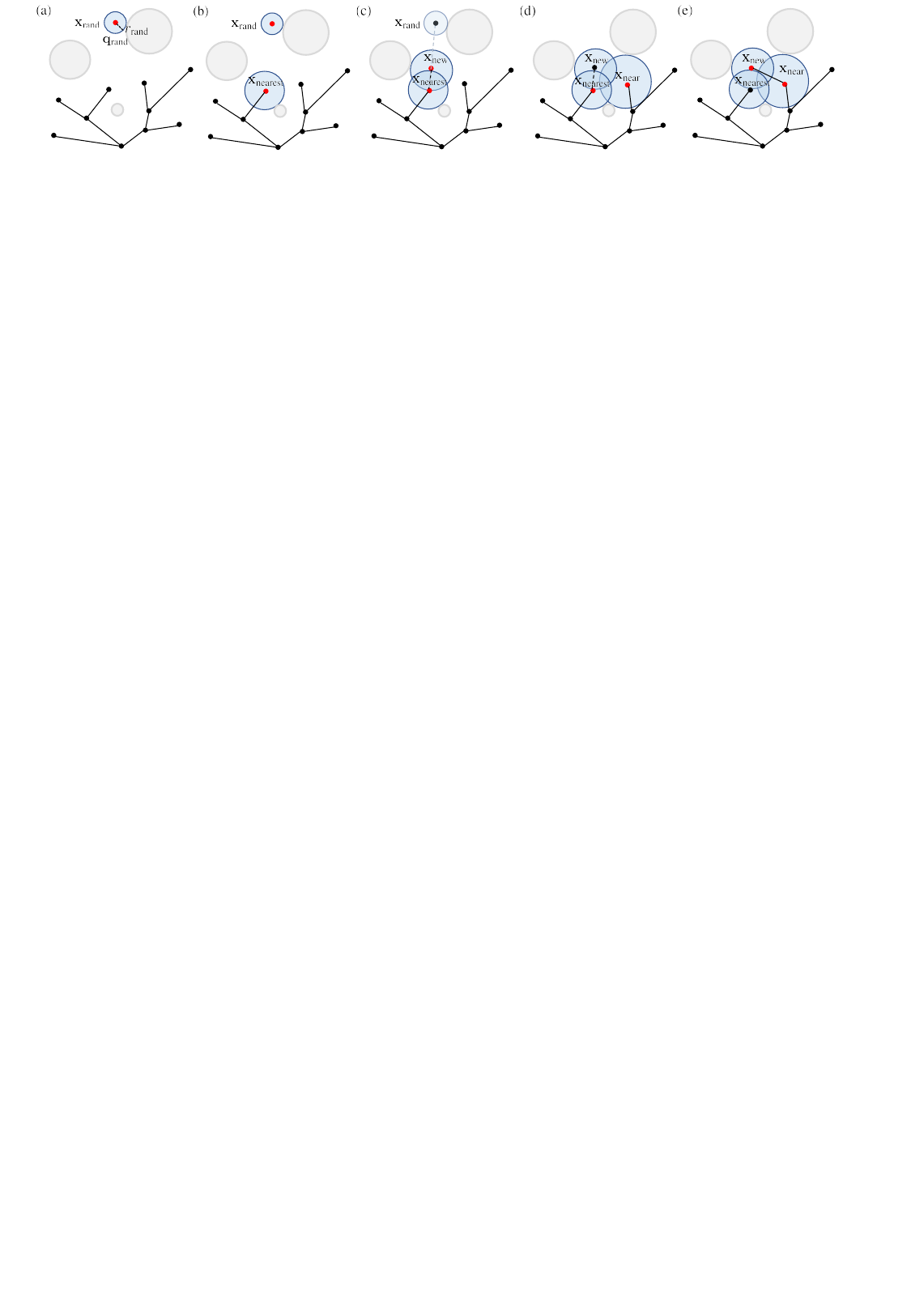}
	\caption{The scheme of Tube RRT* algorithm. Gray circles are obstacles and blue circles are centered in the red and black points. (a) Step 1: Sample a sphere ${\bf x}_\text{rand}$ which is denoted by the blue circle centered in ${\bf q}_\text{rand}$ with radius $r_\text{rand}$ in free space. (b) Step 2: Find the nearest sphere ${\bf x}_\text{nearest}$ in the tree. (c) Step 3: Steer the sample sphere ${\bf x}_\text{rand}$ towards the nearest sphere ${\bf x}_\text{nearest}$ to obtain ${\bf x}_\text{new}$ which has an intersection with ${\bf x}_\text{nearest}$. (d) Step 4: Find all the spheres in the tree which have intersections with the ${\bf x}_\text{new}$ to obtain the set $X_\text{near}$. (e) Step 5: Rewire the sphere ${\bf x}_\text{new}$ to a sphere in the tree to approach the minimum cost. And the spheres in $X_\text{near}$ also are rewired.}
	\label{fig:scheme-of-tuberrt}
	\vspace{-0.5cm}
\end{figure*}

\begin{algorithm}
	\caption{Tube RRT*}
	\begin{algorithmic}[1]
		\REQUIRE start point ${\bf q}_{\text{init}}$, goal point ${\bf q}_{\text{goal}}$, space $\mathcal{X}$
		\ENSURE $\mathcal{G}=\left(V,E\right)$
		\STATE $r_{\text{init}} \leftarrow \tt{FindMaxRadius}\left({\bf q}_{\text{init}}\right)$;
		\STATE ${{\bf{x}}_{{\text{init}}}} \leftarrow \left( {{{\bf{q}}_{{\text{init}}}},{r_{{\text{init}}}}} \right);$
		\STATE $V \leftarrow \left\{ {\bf{x}}_{{\text{init}}} \right\}$; $E \leftarrow \emptyset$; $\mathcal{G}=\left(V,E\right)$;
		\FOR{$i=1,...,n$}
		\STATE ${{\bf{q}}_{{\text{rand}}}} \leftarrow {\tt{SampleFree}}_{{{i}}};$
		\STATE  $r_{\text{rand}} \leftarrow \tt{FindMaxRadius}\left({\bf q}_{\text{rand}}\right)$;
		\STATE ${{\bf{x}}_{{\text{rand}}}} \leftarrow \left( {{{\bf{q}}_{{\text{rand}}}},{r_{{\text{rand}}}}} \right);$
		\STATE ${{\bf{x}}_{{\text{nearest}}}} \leftarrow \tt{Nearest}\left( {{{\bf{q}}_{{\text{rand}}}},{\cal G}} \right);$
		\STATE ${{\bf{x}}_{{\text{new}}}} \leftarrow \tt{TubeSteer}\left( {{{\bf{x}}_{{\text{nearest}}}},{{\bf{x}}_{{\text{rand}}}}} \right);$
		\IF{$r_\text{new}$ $> r_\text{min}$}
		\STATE ${X_{{\text{near}}}} \leftarrow {\tt{NearConnect}}\left( {{\cal G},{{\bf{x}}_{{\text{new}}}}} \right);$
		\STATE $V \leftarrow V   \cup  \left\{ {\bf{x}}_{{\text{new}}} \right\};$
		\STATE ${E} \leftarrow {\tt{Rewire}}\left({\bf x}_\text{new},{X}_\text{near},{E}\right)$;
		\ENDIF
		\ENDFOR
		\RETURN $\mathcal{G}=\left(V,E\right)$;
	\end{algorithmic}
	\label{alg:TubeRRT}
\end{algorithm}
\subsection{Proposed Algorithm}
{The Tube RRT* is a modified version of the RRT* algorithm, which still uniformly samples in $\mathcal{X}_\text{free}$ and modifies functions of steering, finding near nodes, and rewiring to consider both the volumes of sphere-sphere intersections and the tube path length in samples.}
Before the discussion and analysis, the functions in the proposed algorithm that differ from those in RRT* are detailed in the following.

$\tt ObstacleFree$: Connect two centers in spheres to obtain a line segment. Check if there is any intersection between the obstacles and this line segment. If an intersection is detected, return false; otherwise, proceed with true.  The details of the function $\tt ObstacleFree$ is depicted in \emph{Algorithm  \ref{alg:Rewire}}.

$\tt FindMaxRadius$: Find the nearest obstacle point ${\bf q}_\text{obs}$ to the input point ${\bf q}$. Subsequently, the minimum distance $r$ between ${\bf q}_\text{obs}$ and ${\bf q}$ is designated as the radius of a sphere centered at point ${\bf q}$, ensuring the sphere ${\bf x}=\left({\bf q},r\right)$ remains within the free space. Moreover, the radius $r$ must be less than the maximum radius $r_\text{max}$ defined based on environment and swarm conditions.

$\tt TubeSteer:$ The sphere ${\bf x}_\text{rand}$ centered at ${\bf q}_{\text{rand}}$ may not have intersections with the sphere ${\bf x}_\text{nearest}$ centered at ${\bf q}_{\text{nearest}}$, as shown in Fig. \ref{fig:scheme-of-tuberrt}(b). Thus, the sphere ${\bf x}_{\text{rand}}$ needs to be moved to the sphere ${\bf x}_{\text{new}}$ to have an intersection between ${\bf x}_{\text{new}}$ and ${\bf x}_{\text{nearest}}$, as shown in Fig. \ref{fig:scheme-of-tuberrt}(c). The details of the function $\tt TubeSteer$ is depicted in \emph{Algorithm  \ref{alg:TubeSteer}}.

\begin{algorithm}
	\caption{TubeSteer}
	\begin{algorithmic}[1]
		\REQUIRE sample sphere ${\bf x}_{\text{rand}}$, nearest sphere ${\bf x}_{\text{nearest}}$
		\ENSURE updated sphere ${\bf x}_{\text{new}}$
		\STATE $\left({\bf q}_{\text{new} } , {r_{\text{new} }}\right) \leftarrow {\bf x}_{\text{rand}}$;
		$\left({\bf q}_{\text{nearest} } , {r_{\text{nearest} }}\right) \leftarrow {\bf x}_{\text{nearest}}$; 
		\STATE $d \leftarrow \left\| {{{\bf{q}}_{{\text{new}}}} - {{\bf{q}}_{{\text{nearest}}}}} \right\|$;
		\STATE ${\bf t} \leftarrow {{\left( {{{\bf{q}}_{{\text{new}}}} - {{\bf{q}}_{{\text{nearest}}}}} \right)} \mathord{\left/
				{\vphantom {{\left( {{{\bf{q}}_{{\text{new}}}} - {{\bf{q}}_{{\text{nearest}}}}} \right)} {\left\| {{{\bf{q}}_{{\text{new}}}} - {{\bf{q}}_{{\text{nearest}}}}} \right\|}}} \right.
				\kern-\nulldelimiterspace} {\left\| {{{\bf{q}}_{{\text{new}}}} - {{\bf{q}}_{{\text{nearest}}}}} \right\|}}$;
		\WHILE{$r_{\text{new}} + r_{\text{nearest}} \le d$ }
		\STATE $d \leftarrow {\text{max}}\left(r_{\text{new}}, r_{\text{nearest}} \right)$;
		\STATE ${\bf q}_{\text{new}} \leftarrow {\bf q}_{\text{nearest}} + d \cdot {\bf t}$;
		\STATE $r_{\text{new}} \leftarrow \tt{FindMaxRadius}\left({\bf q}_{\text{new}}\right)$;
		\ENDWHILE
		\RETURN ${\bf x}_{\text{new}} \leftarrow \left({\bf q}_{\text{new} } , {r_{\text{new} }}\right) $;
	\end{algorithmic}
	\label{alg:TubeSteer}
\end{algorithm}

${\tt{NearConnect}}$: Find a set of the spheres, denoted by $X_{\text{near}}$, in which all spheres have intersections with the new sphere ${\bf x}_{\text{new}}$. The details are described in \emph{Algorithm \ref{alg:NearConnect}}. 
{It is worth noting that the set \(X_{\text{near}}\) is guaranteed to be a non-empty set with at least one element, which is ensured by \emph{Algorithm  \ref{alg:TubeSteer}}.}

$\tt Rewire$: Rewire the new sphere ${\bf x}_\text{new}$ to the near sphere which has the minimum cost in $X_\text{near}$. And all spheres in $X_\text{near}$ also are rewired, as shown in \emph{Algorithm \ref{alg:Rewire}}.

\begin{algorithm}
	\caption{NearConnect}
	\begin{algorithmic}[1]
		\REQUIRE tree $\mathcal{G}$, new sphere ${\bf x}_{\text{new}}$
		\ENSURE a set of near spheres $X_\text{near}$
		\STATE $\left(V , E\right) \leftarrow \mathcal{G}$, $\left({\bf q}_\text{new},r_\text{new}\right)\leftarrow {\bf x}_\text{new}$;
		\STATE $X_\text{near} \leftarrow \emptyset$;
		\FOR{${\bf x}_\text{near} \in V$}
		\STATE $\left({\bf q}_\text{near},r_\text{near}\right)\leftarrow {\bf x}_\text{near}$;
		\STATE $d \leftarrow \left\| {{{\bf{q}}_{{\text{new}}}} - {{\bf{q}}_{{\text{near}}}}} \right\|$;
		\IF{$r_\text{near} + r_\text{new} > d$}
			\STATE $X_\text{near} \leftarrow X_\text{near} \cup {\bf x}_\text{near} $;
		\ENDIF
		\ENDFOR
		\RETURN $X_\text{near}$;
	\end{algorithmic}
	\label{alg:NearConnect}
\end{algorithm}

${\tt Score}$: A $\tt Score$ function is designed to score a connection between two adjacent spheres, considering both the distance and the volume of sphere-sphere intersections, as depicted in 
\begin{equation}
	{\tt Score} = {\rho _d}\frac{{\left\| {{{\bf{q}}_{{\text{new}}}} - {{\bf{q}}_{{\text{near}}}}} \right\|}}{{\left\| {{{\bf{q}}_{{\text{goal}}}} - {{\bf{q}}_{{\text{init}}}}} \right\|}} + {\rho _v}{\left( {\frac{{{V_{{\mathop{\text{int}}} }}}}{\sigma_v } + \varepsilon } \right)^{ - 1}},
	\label{equ:score}
\end{equation}
where $\rho_v$ and $\rho_d$ are weight coefficients, $V_{\text{int}}$ is the volume of the intersection between spheres of ${\bf q}_\text{new}$ and ${\bf q}_\text{near}$, and $\sigma_v, \varepsilon$ are constants.

\begin{algorithm}
	\caption{Rewire}
	\begin{algorithmic}[1]
		\REQUIRE edges ${E}$, new sphere ${\bf x}_{\text{new}}$, a set of near spheres $X_\text{near}$
		\ENSURE edges ${E}$
			\STATE ${{\bf{x}}_{{\text{min}}}} \leftarrow {{\bf{x}}_{{\text{nearest}}}}\in X_\text{near};$\\${c_{{\text{min}}}} \leftarrow  {\tt Cost}\left( {{{\bf{x}}_{{\text{nearest}}}}} \right) + {\tt{Score}}\left( {{{\bf{x}}_{{\text{nearest}}}},{{\bf{x}}_{{\text{new}}}}} \right);$
			\FOR{${{\bf{x}}_{{\text{near}}}} \in {X_{{\text{near}}}}$}
			\IF{${\tt{ObstacleFree}}\left({\bf{x}}_{{\text{near}}},{\bf{x}}_{{\text{new}}}\right) \wedge {\tt Cost}\left( {{{\bf{x}}_{{\text{near}}}}} \right) + {\tt{Score}}\left( {{{\bf{x}}_{{\text{near}}}},{{\bf{x}}_{{\text{new}}}}} \right) < c_{\text{min}}$}
			\STATE ${{\bf{x}}_{\text{min} }} \leftarrow {{\bf{x}}_{{\text{near}}}};$\\
			${c_{{\text{min}}}} \leftarrow  {\tt Cost}\left( {{{\bf{x}}_{{\text{near}}}}} \right) + {\tt{Score}}\left( {{{\bf{x}}_{{\text{near}}}},{{\bf{x}}_{{\text{new}}}}} \right);$
			\ENDIF
			\ENDFOR
			\STATE $E \leftarrow E \cup \left\{ {\left( {{{\bf{x}}_{{\text{min}}}},{{\bf{x}}_{{\text{new}}}}} \right)} \right\};$
			\FOR{${{\bf{x}}_{{\text{near}}}} \in {X_{{\text{near}}}}$}
			\IF{${\tt{ObstacleFree}}\left({\bf{x}}_{{\text{near}}},{\bf{x}}_{{\text{new}}}\right) \wedge {\tt Cost}\left( {{{\bf{x}}_{{\text{new}}}}} \right) + {\tt{Score}}\left( {{{\bf{x}}_{{\text{near}}}},{{\bf{x}}_{{\text{new}}}}} \right) < {\tt Cost}\left( {{{\bf{x}}_{{\text{near}}}}}\right)$}
			\STATE ${{\bf{x}}_{{\text{parent}}}} \leftarrow {\tt{Parent}}\left( {{{\bf{x}}_{{\text{near}}}}} \right);$\\
			$E \leftarrow \left( {E{{\backslash }}\left\{ {\left( {{{\bf{x}}_{{\text{parent}}}},{{\bf{x}}_{{\text{near}}}}} \right)} \right\}} \right) \cup \left\{ {\left( {{{\bf{x}}_{{\text{new}}}},{{\bf{x}}_{{\text{near}}}}} \right)} \right\};$
			\ENDIF
			\ENDFOR
		\RETURN ${E}$;
	\end{algorithmic}
	\label{alg:Rewire}
\end{algorithm}

The Tube RRT* algorithm, as depicted in \emph{Algorithm \ref{alg:TubeRRT}}, incorporates the addition of a new sphere to the vertex set $V$ while considering both distance and volume of intersection. Consequently, only spheres that intersect with a sphere from the tree are added as new nodes, ensuring that any tree node has intersections with adjacent nodes. During the selection of the parent node and tree rewiring process, the score defined in (\ref{equ:score}) accounts for both distance and volume of intersection between adjacent spheres. This implies that the new sphere tends to select a parent node with a shorter distance and larger intersection volume, resulting in paths generated by the Tube RRT* algorithm having shorter lengths within a broader space.
\subsection{Homotopic Paths Generation}
The generation method of the homotopic paths based on Tube RRT* is described in this subsection. In the optimal virtual tube planning method \cite{mao2023optimal}, the boundary optimal trajectories of the virtual tube based on the boundary paths are generated first. Then, the trajectory within the virtual tube is optimal when its coefficients are interpolated from the boundary coefficients of trajectories, meaning the path points inside the virtual tube are simply interpolations of the boundary path points. Thus, the homotopic path can be generated by the interpolation among boundary paths. The homotopic path ${\boldsymbol{\sigma}}_{\boldsymbol{\theta}}$ within the virtual tube can be expressed as
\begin{equation}
	\begin{array}{rcl}
		{\boldsymbol{\sigma _\theta} } & = & \boldsymbol{\sigma} \left( {\left( {\sum\limits_{k = 1}^M {{\theta _k}{{\bf{q}}_{0,k}}} ,\sum\limits_{k = 1}^M {{\theta _k}{{\bf{q}}_{m,k}}} } \right),t} \right)\\
		& = &\sum\limits_{k = 1}^M {{\theta _k}\boldsymbol{\sigma} \left( {\left( {{{\bf{q}}_{0,k}},{{\bf{q}}_{m,k}}} \right),t} \right)} 
		= \sum\limits_{k = 1}^M {{\theta _k}{\boldsymbol{\sigma} _k}} ,
	\end{array}
	\label{equ:interpolation}
\end{equation}
where $\sum\nolimits_{k = 1}^M {{\theta _k}}  = 1,{\theta _k} \ge 0$, ${\bf q}_{0,k}$ and ${\bf q}_{m,k}$ are vertexes of the terminal ${\mathcal{C}_0}$ and $\mathcal{C}_1$ respectively, $M$ is the number of the vertexes in terminals, $\boldsymbol{\sigma}_k$ are the boundary paths. Combining (\ref{equ:interpolation}) with the process in \emph{Algorithm \ref{alg:TubeRRT}}, the time complexity of $l$ homotopic paths generation is $O\left(n(\text{log}n+1)+l\right)$ where $n$ is the number of samples. Compared with the time complexity $O\left(nl(\text{log}n+1)\right)$ of RRT* \cite{karaman2011sampling}, the time complexity of $l$ homotopic paths in Tube RRT* is independent of the number of samples $n$.

\begin{figure}
	\centering
	\includegraphics{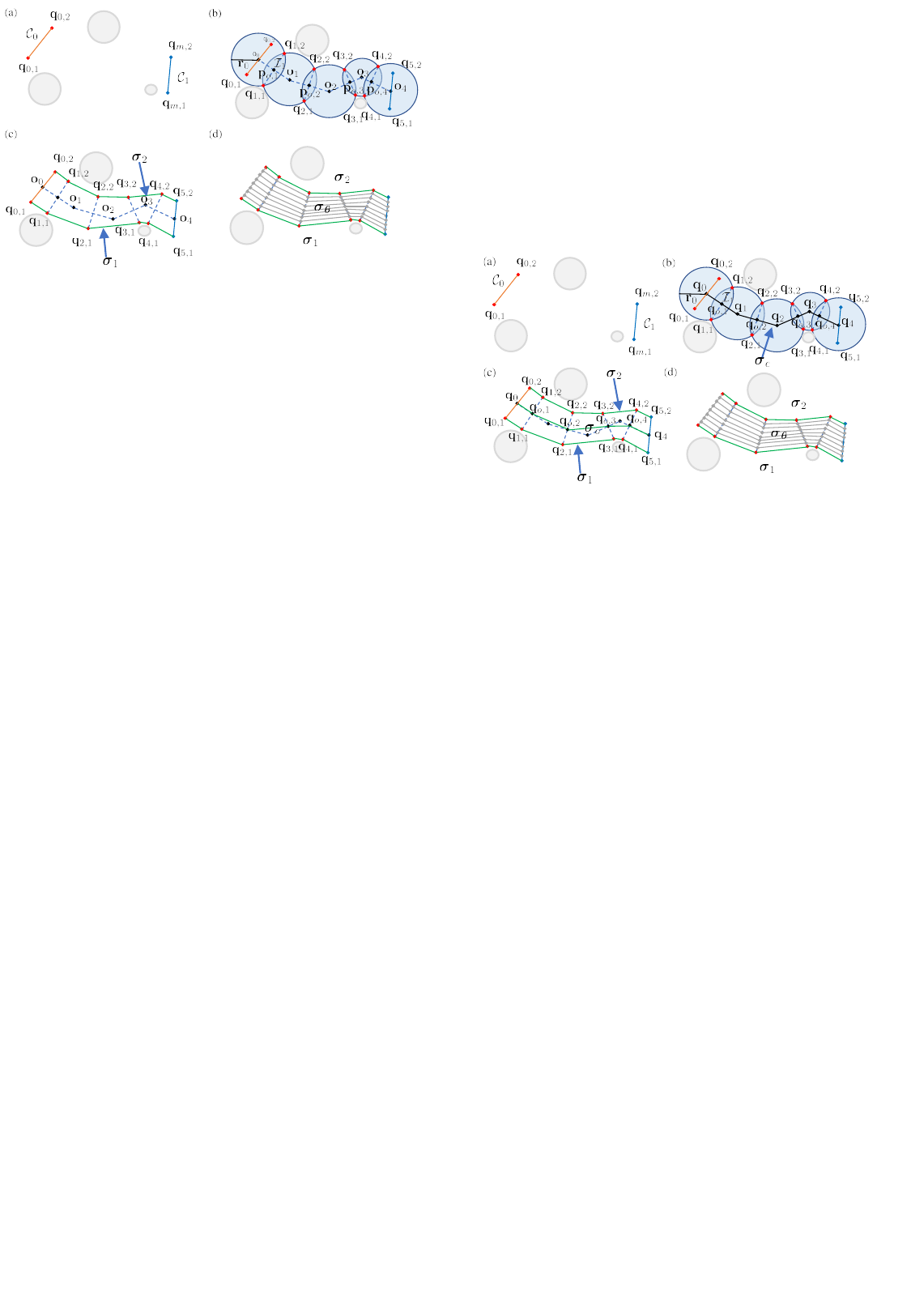}
	\caption{Examples of homotopic paths generation. (a) Red line and blue line represent terminal $\mathcal{C}_0$ and $\mathcal{C}_1$ respectively. (b) Tube RRT* algorithm generates a path $\boldsymbol{{\sigma}}_c$ from ${\bf q}_0$ to ${\bf q}_4$ represented by the black line. (c) Boundary paths $\boldsymbol{\sigma}_1$ and $\boldsymbol{\sigma}_2$ are denoted by green lines. (d) Homotopic paths $\boldsymbol{\sigma}_{\boldsymbol{\theta}}$ are represented by the gray lines.}
	\label{fig:path-point-selection-strategy}
	\vspace{-0.5cm}
\end{figure}
Thus, the key to generating homotopic paths lies in defining the boundary paths. The steps for this process, illustrated in Fig. \ref{fig:path-point-selection-strategy}, are described below.
\begin{itemize}
	\item[1.] The start and goal areas are denoted as terminal ${\mathcal{C}_0}$ and ${\mathcal{C}_1}$, respectively. The vertices of these regions are ${\bf q}_{0,k} \in {\mathcal{C}_0}$ and ${\bf q}_{m,k} \in {\mathcal{C}_1}$, such as ${\bf q}_{0,1}$, ${\bf q}_{0,2}$, ${\bf q}_{m,1}$, and ${\bf q}_{m,2}$, as shown in Fig. \ref{fig:path-point-selection-strategy}(a).
	\item[2.] The Tube RRT* algorithm generates a sequence of intersecting spheres centered at ${\bf q}_{i}$, with radii $r_i \left(i=0,1,...,m-1\right)$. Let $\boldsymbol{{\sigma}}_c$ be the sequence of centers $\{{\bf q}_i\}$. The points ${\bf q}_{o,k}$ are selected at the centers of the spherical intersections between adjacent spheres centered at ${\bf q}_{i-1}$ and ${\bf q}_{i}$. Points ${\bf q}_{i,k}$ lie on the boundaries of these spherical intersections, as shown in Fig. \ref{fig:path-point-selection-strategy}(b).
	\item[3.] The path $\boldsymbol{{\sigma}}_o$ is defined as a sequence of points $\{{\bf q}_{o,i}\} \left(i=0,1,...,m\right)$, where ${\bf q}_{o,0}={\bf q}_0$ and ${\bf q}_{o,m}={\bf q}_{m-1}$, as shown in Fig. \ref{fig:path-point-selection-strategy}(c). The boundary paths $\boldsymbol{\sigma}_k$ are sequences of points $\{{\bf q}_{i,k}\} \left(i=0,1,...,m\right)$.
	\item[4.] Finally, the homotopic path $\boldsymbol{\sigma}_{\boldsymbol{\theta}}$ is generated by interpolating among the boundary paths ${\boldsymbol{\sigma}_k}$. As shown in Fig. \ref{fig:path-point-selection-strategy}(d), the homotopic path $\boldsymbol{\sigma}_{\boldsymbol{\theta}}$ is generated by interpolating between the boundary paths ${\boldsymbol{\sigma}_1}$ and ${\boldsymbol{\sigma}_2}$.
\end{itemize}

\subsection{Analysis}
In this section, the probabilistic completeness and asymptotic optimality of the proposed algorithm are proved, followed by the analysis of some properties.
\textcolor{blue}{
\begin{The}[Probabilistic completeness of Tube RRT*]
		When the number of samples approaches infinity, if a solution exists for the path planning problem, the probability that Tube RRT* will find a feasible solution is 1, namely, 
\[\begin{array}{l}
	\mathop {\lim \inf }\limits_{n \to \infty } P({\boldsymbol{\sigma }_n} \in {\sum _{\boldsymbol{\sigma }_n}},{\boldsymbol{\sigma }_n}\left( 0 \right) = {{\bf{x}}_{{\text{init}}}},{\boldsymbol{\sigma }_n}\left( 1 \right) = {{\bf{x}}_{{\text{goal}}}},\\
	{{\boldsymbol{\sigma }}^*_n}\left( 1 \right) \in {{\cal C}_1}) = 1
\end{array}\]
where $\boldsymbol{\sigma}_n$ is the feasible path found by Tube RRT* with $n$ samples.
\label{the:1}
\end{The}
\begin{proof}
	See details in \emph{Appendix A}.
\end{proof}
\begin{The}[Asymptotic optimality of Tube RRT*]
	When the number of samples approaches infinity, if a solution exists for the path planning problem, the probability that Tube RRT* converges asymptotically
	towards the optimal solution is 1, namely,  
	\[P\left( {\mathop {\lim \sup }\limits_{n \to \infty } {\tt{Cost}}\left( \boldsymbol{\sigma }_n \right) = {\tt{Cost}}\left( {{\boldsymbol{\sigma} ^*_n}} \right)} \right) = 1\]
	where $\boldsymbol{\sigma}_n$ is the feasible path found by Tube RRT* with $n$ samples.
	\label{the:2}
\end{The}
\begin{proof}
	See details in \emph{Appendix B}.
\end{proof}
}
Based on these Theorems, the properties of cost in the proposed algorithm are analyzed. Let ${\boldsymbol{\sigma}}_c = \left\{ {{{\bf{q}}_i}} \right\}\left(i=0,1,...,m-1\right)$ be a collision-free path from start sphere ${\bf x}_\text{init} = \left({\bf q	}_\text{init}, r_\text{init} \right)$ to goal sphere ${\bf x}_\text{goal} = \left({\bf q	}_\text{goal}, r_\text{goal} \right)$. The cost of the path $\boldsymbol{\sigma}_c$ is expressed as
\begin{equation}
	{\tt{Cost}}\left( \boldsymbol{\sigma}_c \right) = \frac{{{\rho _d}}}{{\left\| {{{\bf{q}}_{{\text{goal}}}} - {{\bf{q}}_{{\text{init}}}}} \right\|}}\sum\limits_{i = 1}^{m-1} {{d_i}}  + {\rho _v}\sum\limits_{i = 1}^{m-1} {{{\left( {\frac{{{V_{{\text{int}},i}}}}{\sigma_v} + \varepsilon } \right)}^{ - 1}}},
	\label{equ:cost}
\end{equation}
where $d_i = \| {\bf q}_i - {\bf q}_{i-1} \|$ is the distance between centers of adjacent spheres, $V_{\text{int},i}$ is the intersection volume of the adjacent spheres. It should be noted that the path length in \emph{Definition \ref{def:tube-path}} is represent by $L\left(\boldsymbol{\sigma}_{o}\right)$. However, the path length considered in (\ref{equ:cost}) is for the path $L\left(\boldsymbol{\sigma}_{c}\right)$.
Thus, the following \emph{Lemma} \ref{lemma:2} shows that the length of $\boldsymbol{\sigma}_{o}$ can be approximated by the length of path $\boldsymbol{\sigma}_c$, if the number of samples tends towards infinity.
\begin{lemma}
	For any $\epsilon>0$, there exists $m_c \in \mathbb{Z}^+$ so that when the number of path points $m>m_c$ for the path $\boldsymbol{{\sigma}}_o$ , $L\left(\boldsymbol{\sigma}_c\right) - L\left(\boldsymbol{\sigma}_o\right) \le \epsilon.$
	\label{lemma:2}
\end{lemma}
\begin{proof}
	See details in \emph{Appendix C}.
\end{proof}

According to \emph{Lemma \ref{lemma:2}}, the cost of goal point could be regarded as the a function $f=f_1 +f_2$ with respect to the number of the segments of the path, the total distance of the path, and the total volume of intersection of spheres, which is expressed as
\begin{equation}
	f={\tt{ Cost}}\left( \boldsymbol{\sigma}_o\right).
	\label{equ:costfunction}
\end{equation}
And rewiring function implies that the proposed algorithm intends to find a set of homotopic paths includes an optimal path ${\boldsymbol{{\sigma}}^*_o}$ with a minimum cost in (\ref{equ:costfunction}).
This cost accounts for both path length and intersection volume (gap volume). Between two paths of equal length, the one with a larger intersection volume is preferred. For paths with the same intersection volume, the shorter one is selected. If both length and volume are equal, the path with fewer segments and more uniform intersection volume is favored, as described in \emph{Proposition 1}.
\begin{prop}
	Let $\Pi_{\Sigma}$ be a collection of feasible homotopic paths, where $m_\text{min}$ denotes the minimum number of path points in $\Pi_{\Sigma}$. Suppose all homotopic paths $\Sigma_{\boldsymbol{\sigma}}$ in $\Pi_{\Sigma}$ have equal path lengths and equal total volume $V_\text{total}$. Then, the homotopic paths $\Sigma_{\boldsymbol{\sigma}}$ includes ${\boldsymbol{{\sigma}}^*}$ with the minimum cost is found when the number of path points for ${\boldsymbol{{\sigma}}^*}$ is $m_\text{min}$, and the intersection volumes $V_{{\text{int}},i}$ between any adjacent path points are ${{{V_{{\text{total}}}}} \mathord{\left/ {\vphantom {{{V_{{\text{total}}}}} {{m_{{\text{min}}}}}}} \right. \kern-\nulldelimiterspace} {{m_{{\text{min}}}}}}$.
\end{prop}
\begin{proof}
	See details in \emph{Appendix D}.
\end{proof}

\section{Simulations \textcolor{blue}{and Experiments}}
In this section, the proposed algorithm is compared with the state-of-the-art algorithms to demonstrate its performance. Then, the impact of the parameter in the proposed algorithm is analyzed.
\textcolor{blue}{Finally, several simulations and experiments are conducted to validate the performance of the proposed algorithm.
All the simulations and experiments are implemented on a drone with NIVIDA Jetson Orin. The time limits of all algorithms are set to 0.1s.  
And, the metrics used in comparisons are described in the following.}
\begin{itemize}
	\item Average path length (APL): The average path length is computed by taking the average of all the path lengths $L\left(\boldsymbol{{\sigma}}\right)$ of homotopic paths in the tests.
	\item \textcolor{blue}{Minimum gap volume (MGV): The volume of the smallest gap which the paths passes through is represented by the volume of the smallest sphere among the spheres generated by the path planning algorithm.}
	\item \textcolor{blue}{Variance of safe distance (VSD): The variance of minimum distances to obstacles for path points, representing the degree of change in the cross section of the homotopic paths.}
\end{itemize}
The metrics used in flight simulations are as follows.
\begin{itemize}
	\item \textcolor{blue}{Flight time: The flight time is measured from the start of the drone swarm until the last drone in the swarm reaches the target area.}
	\item \textcolor{blue}{Minimum distance: The minimum distance between drones during the flight time.}
	\item \textcolor{blue}{Flight speed: The speeds of all drones during the flight time.}
\end{itemize}
\subsection{{Performance in Comparisons}}
Tube RRT* is compared with RRT*\cite{karaman2011sampling}, BIT*\cite{gammell2020batch}, and Safe-region RRT*\cite{gao2019flying}, in terms of APL, MGV, and VSD. All algorithms are implemented in C++ and executed on an Ubuntu 20.04 system, with RRT* and BIT* being called from OMPL. The size of environment for comparisons is 25m$\times$25m$\times$3m, with obstacles generated randomly. The start and goal points are consistently set. The hundreds of tests are conducted in environments with 20, 40, 60, and 80 obstacles.
\begin{figure}
	\centering
	\includegraphics{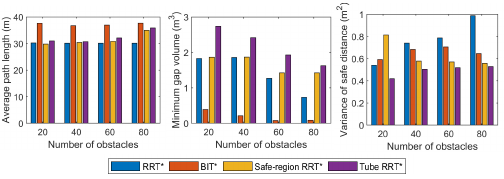}
	\caption{\textcolor{blue}{Comparisons of metrics for algorithms in environments with different number of obstacles.}}
	\label{fig:sr-com}
	\vspace{-0.5cm}
\end{figure}

The statistical results are shown in Fig. \ref{fig:sr-com}. It can be observed that the MGV of Tube RRT* is greater than that of the other methods, indicating that it considers the volume of gaps, which is more conducive to swarm passing through. Additionally, for the volume of gaps in the environment decreasing with more obstacles, the MGV of Tube RRT* decreases as the number of obstacles increases. As for the APL, Tube RRT* is slightly higher than RRT* and Safe-region RRT* but significantly lower than BIT*.
\textcolor{blue}{Meanwhile, the VSD generated by Tube RRT* is the smallest, indicating that the changes of cross-section in homotopic paths are minimal, which helps to avoid unnecessary dispersion and convergence of the swarm.}
Overall, Tube RRT* achieves a significant increase in MGV while keeping the path length increase minimal, which is beneficial for generating homotopic paths and avoiding congestion for the swarm.
\subsection{{Analysis of Parameter Impact}}
{In this subsection, the impact of the parameter \(\rho_v\) in (\ref{equ:cost}) on the performance of Tube RRT* is analyzed.}
{The simulation environment, a map with random obstacles, is reused from our previous work \cite{mao2023optimal}. The size of the map is 250$\times$200$\times$30m, and the size of the random obstacles is 10$\times$10$\times$30m.}
The other parameters in (\ref{equ:cost}) are set as follows: $\rho_d = 1$,  $\sigma_v=1413.7$, and $\epsilon = 0.01$. 
Multiple tests have been performed in different obstacle environments, as shown in Fig. \ref{fig:path-comparsion}. 
And, the metrics are shown in Fig. \ref{fig:pathparam}.

\begin{figure}
	\centering
	\includegraphics{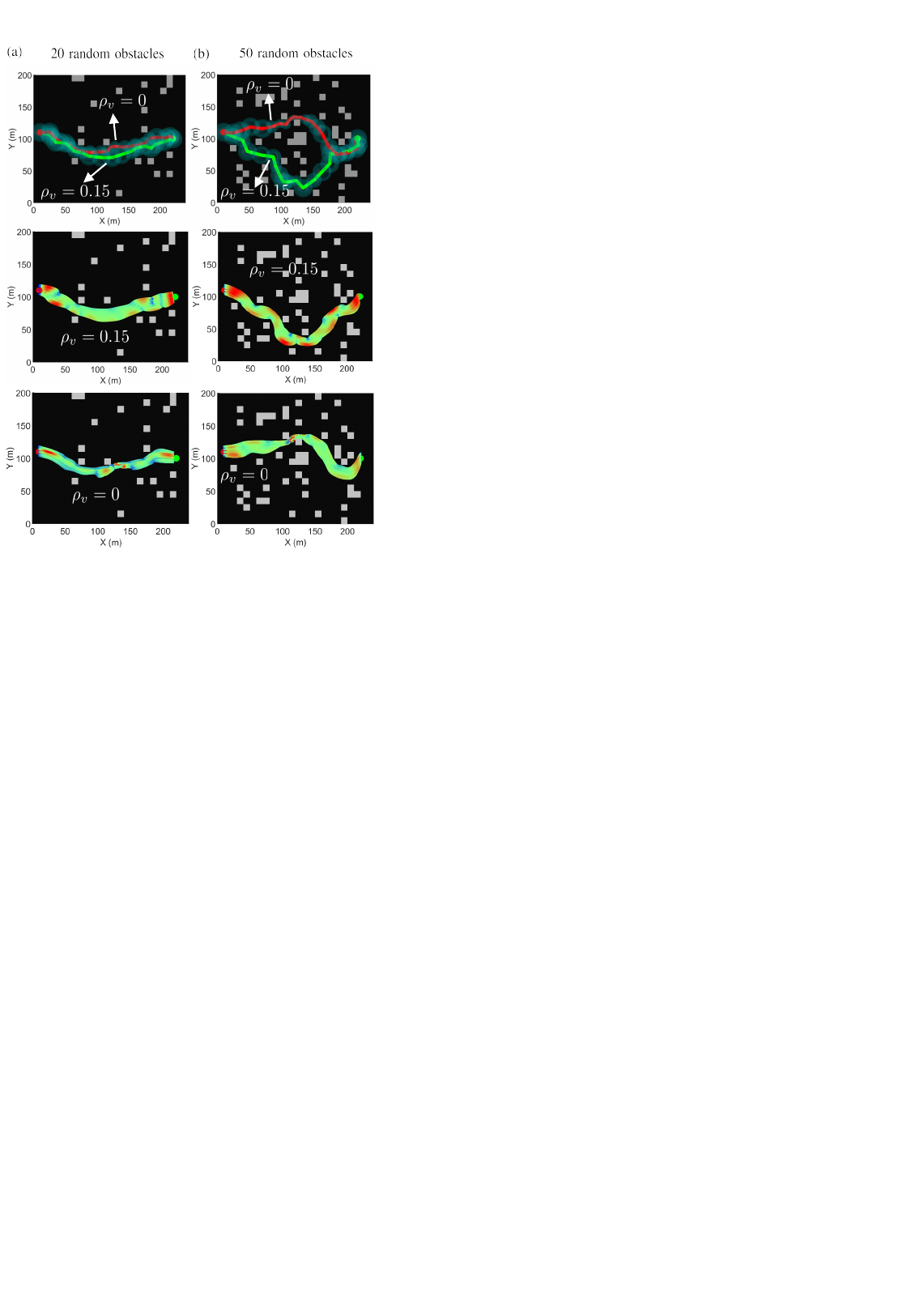}
	\caption{Comparisons of paths generated by Tube RRT*. The red and green points represent the start area and goal area respectively. And the green curves and red curves are generated by Tube RRT* with $\rho_v = 0.15$ and $\rho_v = 0$ respectively. The colorful curves represent trajectories for robots and the colors from blue to red represent the speed from zero to maximum.}
	\label{fig:path-comparsion}
	\vspace{-0.4cm}
\end{figure}
\begin{figure}
	\centering
	\includegraphics{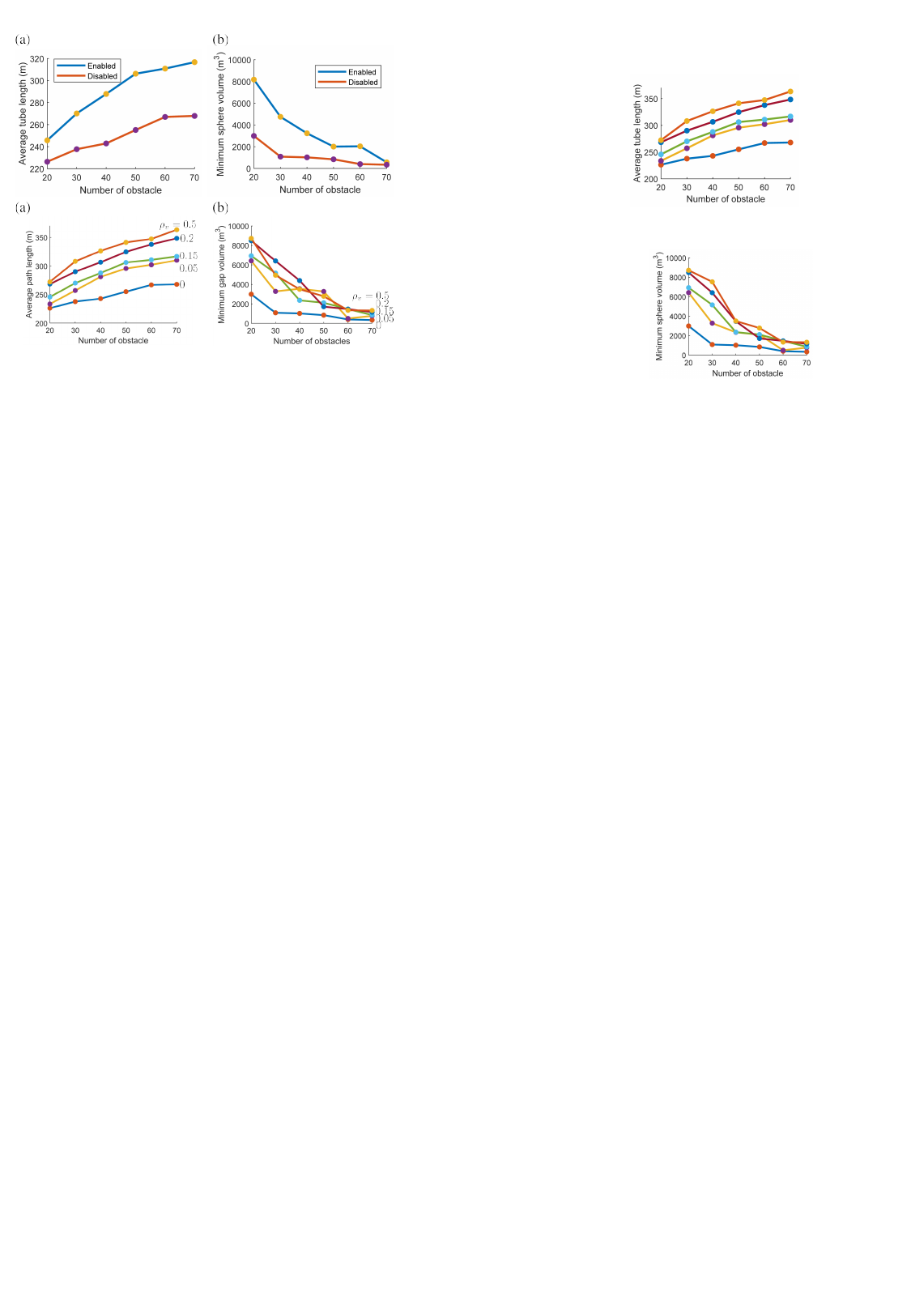}
	\caption{{Metrics for Tube RRT* with varying numbers of obstacles.}}
	\label{fig:pathparam}
	\vspace{-0.7cm}
\end{figure}

{It can be observed that the method with high $\rho_v$ has a relatively larger MGV and a longer APL.}
{When the number of obstacles is small, the APL for all algorithms exhibit little difference, indicating that they generate quite similar paths, as shown in Fig. \ref{fig:path-comparsion}(a). }
{As the number of obstacles increases, the APL of Tube RRT* increases significantly and the MGV of Tube RRT* decreases. Although the MGV of Tube RRT* with high $\rho_v$ is larger, the APL of Tube RRT* with high $\rho_v$ is also larger.}
This suggests that the Tube RRT* algorithm strikes a balance between path length and gap volume, sacrificing the suitable path length in exchange for larger gap size. 
{Therefore, choosing a suitable  \( \rho_v \) would have}
advantageous for the swarm navigation, as it helps to avoid congestion at narrow passages, thereby enhancing traversal efficiency.
However, in scenarios where the number of obstacles is excessively high, the parameter \(\rho_v\) has a minimal impact on MGV and APL, as illustrated in Fig. \ref{fig:pathparam}(b). Consequently, the Tube RRT* algorithm demonstrates more pronounced performance in large-scale, obstacle-dense environments.


\subsection{Simulate Flight Results}
In this subsection, optimal virtual tubes are generated based on the homotopic paths planned by Tube RRT* with $\rho_v=0$ and $\rho_v=0.15$. The distributed controller and planning method remain consistent with \cite{mao2023optimal}. The desired flight time is standardized to 90s across all scenarios.
The swarm consists of 20 drones, with a safety radius of 0.5m and an avoidance radius of 1m. The maximum speed is set to 7m/s.
Through the comparisons of various metrics, including {flight time}, flight speed, and minimum distance among drones, the effectiveness of the proposed method is validated.

When the obstacles in the environment are relatively sparse, the tube paths planned by the two algorithms show minimal differences, resulting in similar flight times, as shown in Table \ref{table:time-comparison}. However, in environments with more obstacles, as shown in Fig. \ref{fig:path-comparsion}(b), the flight results diverge significantly. Specifically, in terms of flight time, the Tube RRT* with $\rho_v=0.15$ yields the result allowing the swarm to reach the target area in a shorter time. Conversely, the virtual tube generated by the Tube RRT* with $\rho_v=0$ navigates through narrow passages, leading to congestion within the swarm. This congestion is evident in the blue area in Fig. \ref{fig:speed-comparison}(a), where the flight speed of the swarm experiences more drastic fluctuations compared with the virtual tube generated by Tube RRT* shown in Fig. \ref{fig:speed-comparison}(b), resulting in increased collision risk and reduced safety. As shown in Fig. \ref{fig:speed-comparison}(c), the proximity of the drone exceeds the safe distance threshold when traversing narrow gaps.
{Although the virtual tube based on Tube RRT* with $\rho_v=0.15$ has more turns, leading to greater speed variance outside the blue block, as shown in Fig. \ref{fig:speed-comparison}, compared to the virtual tube based on Tube RRT* with $\rho_v=0$, it avoids congestion and thus improves traversal efficiency.}
{When the number of obstacles in the environment is extremely dense, Tube RRT* with $\rho_v=0.15$ does not perform better in terms of metrics in the above subsection, resulting in a longer flight time, as shown in the case of 70 obstacles of Table \ref{table:time-comparison}.}

\begin{figure}
	\centering
	\includegraphics{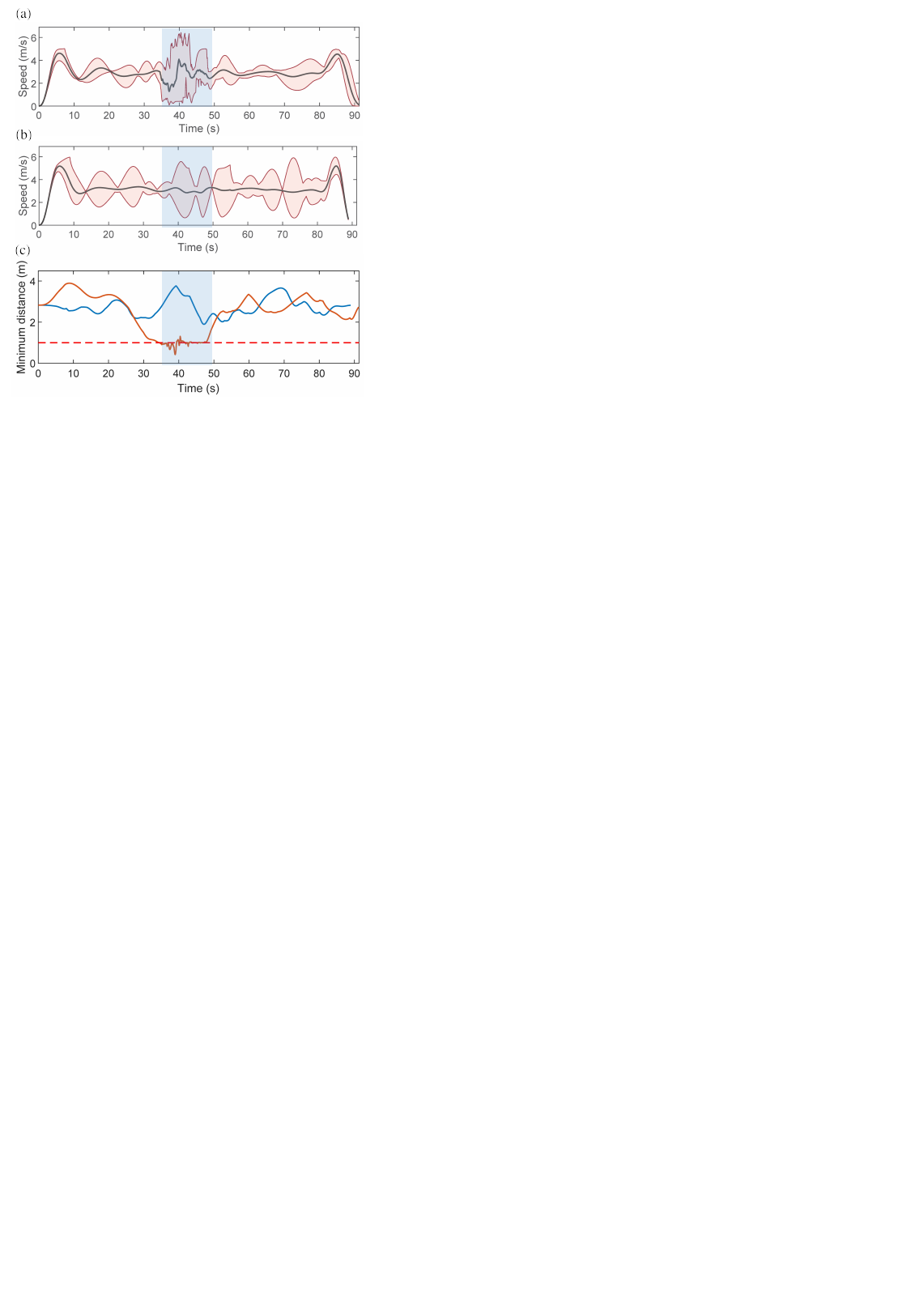}
	\caption{Swarm speed distributions and the minimum distances among drones over time in a map with 50 random obstacles. For (a) and (b), the red curves and black curves represent the maximum speed, minimum speed, and average speed of the swarm. (a) Speed distributions in the tube based on the Tube RRT* with $\rho_v=0$. The blue block represents the segment of passing through narrow gaps. (b) Speed distributions in the tube based on the Tube RRT* with $\rho_v=0.15$. (c) The red curve and blue curve represent the minimum distance in tubes generated by the Tube RRT* with $\rho_v=0$ and $\rho_v=0.15$ respectively. And, the red line represents the safety distance.}
	\label{fig:speed-comparison}
	\vspace{-0.4cm}
\end{figure}


\begin{table}
	\centering
	\begin{tabular}{|l|c|c|c|c|c|c|}
		\hline
		{$\rho_v$} & 20 & 30 & 40& 50  & 60& {70} \\ \hline
		{0}      & 90.1s & 93.4s & 96.8s & 91.3s & 92.8s&{99.5s} \\ \hline
		{0.15} & 90.1s & 89.0s & 89.5s & 89.0s & 88.9s&{106.8s }\\ \hline
	\end{tabular}
	\caption{Flight times of methods in maps with various obstacles.}
	\label{table:time-comparison}
	\vspace{-1cm}
\end{table}
\begin{figure*}[h]
	\centering
	\setlength{\abovecaptionskip}{-0.4cm}
	\includegraphics{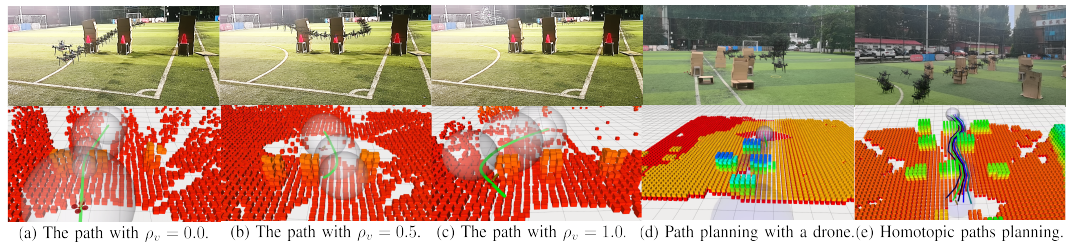}
	\caption{\textcolor{blue}{The real flight experiments of real-time path planning generated by Tube RRT*. }}
	\label{fig:real-flight}
	\vspace{-0.5cm}
\end{figure*}
\subsection{\textcolor{blue}{Experiments}}
\textcolor{blue}{To validate the actual performance of the proposed algorithm, the experimental setup includes drones equipped with the onboard computer (NVIDIA Jetson Orin) and LiDAR (Mid-360). The point cloud of obstacles is obtained from LiDAR, localization and the proposed path planning algorithm are implemented on NVIDIA Jetson Orin. 
First, as shown in Figure \ref{fig:real-flight}(a)-(c), the impact of the parameter $\rho_v$ on the path was verified. Then, in an environment with random obstacles, the real-time performance of the proposed method was validated through a drone quickly navigating the unknown environment, as shown in Figure \ref{fig:real-flight}(d). And the homotopic path planning was verified by sharing the homotopic paths among three drones, as shown in Figure \ref{fig:real-flight}(e).
}
\vspace{-0.3cm}
\section{Conclusion and Future Work}
The Tube RRT* algorithm is introduced for planning homotopic paths, distinguished by its simultaneous consideration of the gap volume and the path length.  
Through simulations and experiments, the proposed algorithm has been substantiated to significantly improve the capability of swarm to navigate obstacle environments. 
\textcolor{blue}{However, homotopic path planning in high-dimensional spaces, dynamic obstacle environments, or extremely dense and complex environments requires further research in future work.}
\addtolength{\textheight}{-0cm}   




%


\bibliographystyle{bib/IEEEtran} 
\bibliography{bib/IEEEabrv,bib/IEEEbib}

\begin{thebibliography}{10}
\providecommand{\url}[1]{#1}
\csname url@rmstyle\endcsname
\providecommand{\newblock}{\relax}
\providecommand{\bibinfo}[2]{#2}
\providecommand\BIBentrySTDinterwordspacing{\spaceskip=0pt\relax}
\providecommand\BIBentryALTinterwordstretchfactor{4}
\providecommand\BIBentryALTinterwordspacing{\spaceskip=\fontdimen2\font plus
\BIBentryALTinterwordstretchfactor\fontdimen3\font minus
  \fontdimen4\font\relax}
\providecommand\BIBforeignlanguage[2]{{%
\expandafter\ifx\csname l@#1\endcsname\relax
\typeout{** WARNING: IEEEtran.bst: No hyphenation pattern has been}%
\typeout{** loaded for the language `#1'. Using the pattern for}%
\typeout{** the default language instead.}%
\else
\language=\csname l@#1\endcsname
\fi
#2}}

\bibitem{zhou2022swarm}
X.~Zhou, X.~Wen, Z.~Wang, Y.~Gao, H.~Li, Q.~Wang, T.~Yang, H.~Lu, Y.~Cao,
  C.~Xu, \emph{et~al.}, ``{Swarm of Micro Flying Robots in the Wild},''
  \emph{Science Robotics}, vol.~7, no.~66, 2022.

\bibitem{vasarhelyi2018optimized}
G.~V{\'a}s{\'a}rhelyi, C.~Vir{\'a}gh, G.~Somorjai, T.~Nepusz, A.~E. Eiben, and
  T.~Vicsek, ``{Optimized Flocking of Autonomous Drones in Confined
  Environments},'' \emph{Science Robotics}, vol.~3, no.~20, 2018.

\bibitem{Quan2021Distributed}
Q.~Quan, Y.~Gao, and C.~Bai, ``{Distributed Control for a Robotic Swarm to Pass
  Through a Curve Virtual Tube},'' \emph{Robotics and Autonomous Systems}, vol.
  162, 2023.

\bibitem{usenko_real-time_2017}
V.~Usenko, L.~von Stumberg, A.~Pangercic, and D.~Cremers, ``{Real-time
  Trajectory Replanning for {MAVs} using Uniform {B}-splines and A {3D}
  Circular Buffer},'' in \emph{2017 {IEEE}/{RSJ} {International} {Conference}
  on {Intelligent} {Robots} and {Systems} ({IROS})}.\hskip 1em plus 0.5em minus
  0.4em\relax IEEE, 2017, pp. 215--222.

\bibitem{ding_efficient_2019}
W.~Ding, W.~Gao, K.~Wang, and S.~Shen, ``An {Efficient} {B}-{Spline}-{Based}
  {Kinodynamic} {Replanning} {Framework} for {Quadrotors},'' \emph{IEEE
  Transactions on Robotics}, vol.~35, no.~6, pp. 1287--1306, 2019.

\bibitem{mao2023optimal}
P.~Mao, R.~Fu, and Q.~Quan, ``{Optimal Virtual Tube Planning and Control for
  Swarm Robotics},'' \emph{The International Journal of Robotics Research},
  vol.~43, no.~5, pp. 602--627, 2024.

\bibitem{hernandez2015comparison}
E.~Hernandez, M.~Carreras, and P.~Ridao, ``{A Comparison of Homotopic Path
  Planning Algorithms for Robotic Applications},'' \emph{Robotics and
  Autonomous Systems}, vol.~64, pp. 44--58, 2015.

\bibitem{kim2003motion}
J.~Kim and J.~P. Ostrowski, ``{Motion Planning a Aerial Robot Using
  Rapidly-Exploring Random Trees with Dynamic Constraints},'' in \emph{2003
  IEEE International Conference on Robotics and Automation ({ICRA})},
  vol.~2.\hskip 1em plus 0.5em minus 0.4em\relax IEEE, 2003, pp. 2200--2205.

\bibitem{hart1968formal}
P.~E. Hart, N.~J. Nilsson, and B.~Raphael, ``{A Formal Basis for the Heuristic
  Determination of Minimum Cost Paths},'' \emph{IEEE Transactions on Systems
  Science and Cybernetics}, vol.~4, no.~2, pp. 100--107, 1968.

\bibitem{yi2016homotopy}
D.~Yi, M.~A. Goodrich, and K.~D. Seppi, ``{Homotopy-Aware RRT*: Toward
  Human-Robot Topological Path-Planning},'' in \emph{2016 11th ACM/IEEE
  International Conference on Human-Robot Interaction (HRI)}.\hskip 1em plus
  0.5em minus 0.4em\relax IEEE, 2016, pp. 279--286.

\bibitem{liu2023homotopy}
J.~Liu, M.~Fu, A.~Liu, W.~Zhang, and B.~Chen, ``{A Homotopy Invariant Based on
  Convex Dissection Topology and a Distance Optimal Path Planning Algorithm},''
  \emph{IEEE Robotics and Automation Letters}, vol.~8, no.~11, pp. 7695--7702,
  2023.

\bibitem{fu2023ftsa}
J.~Fu, W.~Yao, G.~Sun, H.~Tian, and L.~Wu, ``{An FTSA Trajectory Elliptical
  Homotopy for Unmanned Vehicles Path Planning with Multi-Objective
  Constraints},'' \emph{IEEE Transactions on Intelligent Vehicles}, vol.~8,
  no.~3, pp. 2415--2425, 2023.

\bibitem{osa2022motion}
T.~Osa, ``{Motion Planning by Learning the Solution Manifold in Trajectory
  Optimization},'' \emph{The International Journal of Robotics Research},
  vol.~41, no.~3, pp. 281--311, 2022.

\bibitem{karaman2011sampling}
S.~Karaman and E.~Frazzoli, ``{Sampling-based Algorithms for Optimal Motion
  Planning},'' \emph{The International Journal of Robotics Research}, vol.~30,
  no.~7, pp. 846--894, 2011.

\bibitem{gammell2020batch}
J.~D. Gammell, T.~D. Barfoot, and S.~S. Srinivasa, ``{Batch Informed Trees
  (BIT*): Informed Asymptotically Optimal Anytime Search},'' \emph{The
  International Journal of Robotics Research}, vol.~39, no.~5, pp. 543--567,
  2020.

\bibitem{gao2019flying}
F.~Gao, W.~Wu, W.~Gao, and S.~Shen, ``{Flying on Point Clouds: Online
  Trajectory Generation and Autonomous Navigation for Quadrotors in Cluttered
  Environments},'' \emph{Journal of Field Robotics}, vol.~36, no.~4, pp.
  710--733, 2019.

\bibitem{lavalle2001randomized}
S.~M. LaValle and J.~J. Kuffner~Jr, ``{Randomized Kinodynamic Planning},''
  \emph{The International Journal of Robotics Research}, vol.~20, no.~5, pp.
  378--400, 2001.

\end{thebibliography}
\vspace{-0.35cm}
\section*{APPENDIX}
\subsection{The Proof of Theorem \ref{the:1}}
\textcolor{blue}{
The probabilistic completeness of RRT has been proven in \cite{lavalle2001randomized}. Let the graphs $G\left(V,E\right)$ construction by RRT and Tube RRT* be $G_n^{RRT}\left(V_n^{RRT},E_n^{RRT}\right)$ and $G_n^{Tube}\left(V_n^{Tube},E_n^{Tube}\right)$ respectively. By the same sampling distribution, $V_n^{RRT} = V_n^{Tube}$. Moreover, Tube RRT*, like RRT, returns a connected graph, more specifically, a tree. Hence, the
result follows directly from the probabilistic completeness of RRT.
}
\subsection{The Proof of Theorem \ref{the:2}}
\textcolor{blue}{In the proof of asymptotic optimality for RRT* (Lemma 71 in \cite{karaman2011sampling}) , the connection radius $r_n$ must be greater than $4{\left( {\frac{{\mu \left( {{X_{{\text{free}}}}} \right)}}{{{\zeta _d}}}} \right)^{\frac{1}{d}}}\frac{{\log \left( {{\text{card}}\left( V \right)} \right)}}{{{\text{card}}{{\left( V \right)}^{\frac{1}{d}}}}}$. Thus, if the radius of the sphere in the proposed algorithm is greater than $4{\left( {\frac{{\mu \left( {{X_{{\text{free}}}}} \right)}}{{{\zeta _d}}}} \right)^{\frac{1}{d}}}\frac{{\log \left( {{\text{card}}\left( V \right)} \right)}}{{{\text{card}}{{\left( V \right)}^{\frac{1}{d}}}}}$, the asymptotic optimality is obtained.}
\subsection{The Proof of Lemma \ref{lemma:2}}
The length of $\boldsymbol{\sigma}_c$ shown in Fig. \ref{fig:path-point-selection-strategy}(b) is expressed as 
\begin{equation}
	\begin{array}{rcl}
		L\left( {{{\boldsymbol{\sigma }}_c}} \right) &=& \sum\nolimits_{i = 1}^{m-1} {\left\| {{{\bf{q}}_i} - {{\bf{q}}_{i - 1}}} \right\|} \\
		&=& \sum\nolimits_{i = 1}^{m-1} {\left( {\left\| {{{\bf{q}}_i} - {{\bf{q}}_{o,i}}} \right\| + \left\| {{{\bf{q}}_{o,i}} - {{\bf{q}}_{i - 1}}} \right\|} \right)} .
	\end{array}
	\label{equ:length-of-center}
\end{equation}
And the length of $\boldsymbol{\sigma}_o$ is expressed as 
\begin{equation}
	L\left(\boldsymbol{\sigma}_o\right) = \sum\nolimits_{i = 1}^{m} {\left\| {{{\bf{q}}_{o,i}} - {{\bf{q}}_{o,i - 1}}} \right\|}, {\bf q}_{o,0} = {\bf q}_0, {\bf q}_{o,m} = {\bf q}_{m-1}.
	\label{equ:length-of-middle} 
\end{equation}
Combining the triangle inequality
\begin{equation}
	\left\| {{{\bf{q}}_{o,i}} - {{\bf{q}}_{o,i - 1}}} \right\| \le \left\| {{{\bf{q}}_{i}} - {{\bf{q}}_{o,i}}} \right\| + \left\| {{{\bf{q}}_{i - 1}} - {{\bf{q}}_{o,i}}} \right\|
	\label{equ:triangles-inequ}
\end{equation}
with the uniform random sampling, for any $\epsilon>0$, there exists $m_c\in \mathbb{Z}^+$ such that when $m>m_c$, 
\begin{equation}
	\max \left\{ {\left\| {{{\bf{q}}_i} - {{\bf{q}}_{o,i}}} \right\|} \right\} \le \frac{\varepsilon }{m}.
	\label{equ:random-distance}
\end{equation}
Thus, combining (\ref{equ:length-of-center}), (\ref{equ:length-of-middle}), (\ref{equ:triangles-inequ}) with (\ref{equ:random-distance}), it can be obtained that 
\begin{equation}
	L\left(\boldsymbol{\sigma}_c\right) - L\left(\boldsymbol{\sigma}_o\right) \le \epsilon.
\end{equation}

\subsection{The Proof of Proposition 1}
Let $V_\text{total}$ be the total intersection volume of the path. The problem which satisfies KKT condition can be expressed as a convex optimization problem
\begin{subequations}
	\begin{align}
		\mathop {\min }\limits_{{V_{{\text{int}},i}},m} {\text{  }}& \sum\nolimits_{i = 1}^{m - 1} {{{\left( {{{{V_{{\text{int}},i}}} \mathord{\left/
								{\vphantom {{{V_{{\text{int}},i}}} {{\sigma _v}}}} \right.
								\kern-\nulldelimiterspace} {{\sigma _v}}} + \varepsilon } \right)}^{ - 1}}} \label{equ:prop1a}\\
		{\text{s}}.{\text{t}}.{\text{ }}&\sum\nolimits_{i = 1}^n {{V_{{\text{int}},i}}}  = {V_{{\text{total}}}},  \label{equ:prop1b}\\
		&{V_{{\text{int}},i}} > 0,i = 1,...,m-1,\label{equ:prop1c}\\
		&m \ge m_\text{min}\in \mathbb{Z}^{+}.
	\end{align}
	\label{equ:prop1}
\end{subequations}
A step-by-step approach to optimization is used. First, variables $V_{\text{int},i}$ are optimized. Then, optimal $m$ is obtained. 

The Lagrange function of problem (\ref{equ:prop1}) is obtained by
\begin{equation}
	\begin{array}{l}
		L\left( {{V_{{\text{int}},i}},\lambda ,{\mu _i}} \right){\text{ }} = \sum\nolimits_{i = 1}^{m - 1} {{{\left( {{{{V_{{\text{int}},i}}} \mathord{\left/
								{\vphantom {{{V_{{\text{int}},i}}} {{\sigma _v}}}} \right.
								\kern-\nulldelimiterspace} {{\sigma _v}}} + \varepsilon } \right)}^{ - 1}}}  + \\ 
		\lambda \left( {\sum\nolimits_{i = 1}^{m - 1} {{V_{{\text{int}},i}}}  - {V_{{\text{total}}}}} \right) - \sum\nolimits_{i = 1}^{m - 1} {{\mu _i}{V_{{\text{int}},i}}} .
	\end{array}
\end{equation}
Thus, KKT condition contains
\begin{subequations}
	\begin{align}
		{\nabla _{{V_{{\text{int}},i}}}}L =  - \sigma _v^{ - 1}{\left( {{{{V_{{\text{int}},i}}} \mathord{\left/
						{\vphantom {{{V_{{\text{int}},i}}} {{\sigma _v}}}} \right.
						\kern-\nulldelimiterspace} {{\sigma _v}}} + \varepsilon } \right)^{ - 2}} + &\lambda  - {\mu _i} = 0,\label{equ:kkt-a}\\
		\sum\nolimits_{i = 1}^{m - 1} {{V_{{\text{int}},i}}}  - {V_{{\text{total}}}} &= 0,\label{equ:kkt-b}\\
		 -  {V_{{\text{int}},i}} & \le 0,\\
		 {\mu _i} &\ge 0,\label{equ:kktd}\\
		 {{\mu _i}{V_{{\text{int}},i}}} = 0, &i = 1,...,m-1.\label{equ:kkte}
	\end{align}
\end{subequations}
Combining (\ref{equ:prop1c}), (\ref{equ:kktd}) and (\ref{equ:kkte}), it is obtained that $\mu_i = 0$, so that (\ref{equ:kkt-a}) can be expressed as 
\begin{equation}
	\lambda  = \sigma _v^{ - 1}{\left( {{{{V_{{\text{int}},i}}} \mathord{\left/
					{\vphantom {{{V_{{\text{int}},i}}} {{\sigma _v}}}} \right.
					\kern-\nulldelimiterspace} {{\sigma _v}}} + \varepsilon } \right)^{ - 2}},i = 1,...,m - 1.
	\label{equ:equal=vol}
\end{equation}
By substituting (\ref{equ:equal=vol}) into (\ref{equ:kkt-b}), it is obtained that
\begin{equation}
	V_{{\text{int}},i}^* = {{{V_{{\text{total}}}}} \mathord{\left/
			{\vphantom {{{V_{{\text{total}}}}} n}} \right.
			\kern-\nulldelimiterspace} n},i = 1,...,m - 1.
	\label{equ:equal-volume}
\end{equation}

In the next step, substitute  (\ref{equ:equal-volume}) into (\ref{equ:prop1}) to obtain a new optimization problem which is expressed as
\begin{subequations}
	\begin{align}
		\mathop {\min }\limits_{m\in \mathbb{Z}^+} \frac{{{{\left(m-1\right)}^2}}}{{\frac{{{V_{{\text{total}}}}}}{\sigma_v } + \varepsilon {\left(m-1\right)}}}\label{equ:min-n-1}\\
		\text{s.t. } m\ge m_\text{min}.
	\end{align}
	\label{equ:min-n}
\end{subequations}
The objective function in (\ref{equ:min-n-1}) is monotonically decreasing with respect to $m$. Thus, the solution to (\ref{equ:min-n}) is $m_\text{min}$, namely, the minimum number of path points.

\end{document}